\documentclass[a4paper,11pt]{article}
\pdfoutput=1
\usepackage{a4wide}
\usepackage{amsmath,amsfonts,amssymb,amsthm}
\usepackage[colorlinks,citecolor=blue]{hyperref}
\usepackage{enumitem}
\usepackage{bbm}
\usepackage{natbib}
\usepackage{microtype}
\usepackage{graphicx}
\usepackage{extarrows}
\usepackage{pifont}
\usepackage{booktabs}

\setcitestyle{square}

\newtheorem{theorem}{Theorem}[section]
\newtheorem{lemma}[theorem]{Lemma}

\theoremstyle{definition}
\newtheorem{definition}[theorem]{Definition}

\newtheorem{remark}[theorem]{Remark}

\DeclareMathOperator*{\argmin}{argmin}

\numberwithin{equation}{section}
\numberwithin{table}{section}
\numberwithin{figure}{section}

\def \bD {\mathbb{D}}
\def \bE {\mathbb{E}}

\def \bN {\mathbb{N}}

\def \bP {\mathbb{P}}
\def \bQ {\mathbb{Q}}
\def \bR {\mathbb{R}}

\def \bZ {\mathbb{Z}}

\def \cA {\mathcal{A}}
\def \cB {\mathcal{B}}

\def \cD {\mathcal{D}}
\def \cE {\mathcal{E}}
\def \cF {\mathcal{F}}
\def \cG {\mathcal{G}}
\def \cH {\mathcal{H}}

\def \cM {\mathcal{M}}
\def \cN {\mathcal{N}}
\def \cO {\mathcal{O}}
\def \cP {\mathcal{P}}

\def \cR {\mathcal{R}}
\def \cS {\mathcal{S}}
\def \cT {\mathcal{T}}

\def \cX {\mathcal{X}}

\def \NN {\mathcal{NN}}

\def \Pdim {\,{\rm Pdim}\,}
\def \sgn {\,{\rm sgn}\,}
\def \star {\,{\rm star}\,}

\begin{document}
\title{Minimax rates for learning spectral Barron functions by deep ReLU neural networks}
\author{
Songqiu Ma \thanks{School of Mathematics (Zhuhai), Sun Yat-sen University, Zhuhai, China. E-mail: \href{mailto:masq3@mail2.sysu.edu.cn}{masq3@mail2.sysu.edu.cn}.}
\and
Yunfei Yang \thanks{School of Mathematics (Zhuhai) and Guangdong Province Key Laboratory of Computational Science, Sun Yat-sen University, Zhuhai, China. Corresponding author, E-mail: \href{mailto:yangyunfei@mail.sysu.edu.cn}{yangyunfei@mail.sysu.edu.cn}.}
}
\date{}
\maketitle

\begin{abstract}
We study how well deep neural networks approximate and learn spectral Barron functions. Recent studies have shown that these function classes can be efficiently approximated by shallow neural networks without suffering from the curse of dimensionality. We complement these results by providing new approximation bounds for deep networks with ReLU activation and establishing the minimax rates for learning these function classes. Specifically, we show that $d$-dimensional spectral Barron functions with smoothness index $s>0$ can be approximated by deep ReLU neural networks with approximation rate $\widetilde{\mathcal{O}} (S^{-\frac{1}{2}-\frac{s}{d}})$, where $S$ denotes the number of nonzero parameters in the network. Using this approximation result, we further show that deep ReLU neural networks can learn spectral Barron functions in a fast rate $n^{-\frac{d+2s}{2d+2s}}$ with $n$ training samples. Finally, we prove that this convergence rate is minimax optimal up to logarithmic factors.

\end{abstract}

\section{Introduction}

In the past decade, deep learning has achieved unprecedented empirical success in processing high-dimensional data across various fields \citep{lecun2015deep,goodfellow2016deep}. While these practical triumphs highlight the remarkable capacity of deep neural networks to extract intricate patterns from complex data, a rigorous theoretical foundation explaining why they work so well is still in its infancy. Bridging this gap between practice and theory is of paramount importance: it not only provides provable performance guarantees but also guides the principled design of network architectures. 

One of the key reasons for the success of deep neural networks is their ability to effectively approximate and learn many complex nonlinear functions. It is well-known that neural networks with one hidden layer are universal in the sense that they can approximate any continuous functions on compact sets \citep{cybenko1989approximation,hornik1991approximation,pinkus1999approximation}. Recent studies also characterize how well neural networks with ReLU activation approximate classical smooth function classes, such as H\"older, Sobolev and Besov spaces, by the number of parameters in the network \citep{yarotsky2017error,yarotsky2018optimal,yarotsky2020phase,suzuki2019adaptivity} and the width and depth \citep{shen2020deep,shen2022optimal,lu2021deep,siegel2023optimal,yang2025optimalshallow,yang2025optimal,li2026shallow}. These approximation results can be used to show that deep neural networks can achieve minimax optimal rates (up to logarithmic factors) for learning smooth function classes in nonparametric regression \citep{schmidthieber2020nonparametric,suzuki2019adaptivity,kohler2021rate,fan2024how,yang2024nonparametric,yang2025optimal}. However, the approximation and learning rates of these function classes unavoidably suffer from the curse of dimensionality due to the high complexities of these function spaces. Hence, these results cannot explain the practical success of deep learning on high-dimensional data. In order to overcome or weaken the curse of dimensionality, several structural assumptions on the data distribution and the target function have been considered in the literature. For instance, one may assume that the data distribution is low-dimensional \citep{nakada2020adaptive,chen2022nonparametric,jiao2023deep}, the target function has compositional structure \citep{schmidthieber2020nonparametric,kohler2021rate,fan2024how,huang2026learning} or is anisotropic and mixed smooth \citep{suzuki2019adaptivity,montanelli2019new,suzuki2021deep,yang2026approximation}.

Another typical example that we can avoid the curse of dimensionality is when the target function is in the closed convex hull of some bounded dictionary $\bD$, i.e.,
\[
f\in B_1(\bD) := \overline{\left\{ \sum_{i=1}^N a_i d_i,\ d_i\in \bD,\ N\in\bN,\ \sum_{i=1}^N |a_i|\le 1 \right\}}.
\]
One can characterize these functions by using the gauge norm $\|f\|_\bD := \inf\{ c>0, f\in c B_1(\bD)\}$. This norm naturally defines a Banach space, which is called the variation space of the dictionary \citep{kurkova2001bounds,kurkova2002comparison,siegel2022sharp,siegel2023characterization} and $B_1(\bD)$ is the unit ball of this space. When the dictionary $\bD$ is parameterized by the activation function in neural network, these spaces are also known as Barron spaces \citep{barron1993universal,klusowski2018approximation,weinan2019priori,weinan2022barron}. For the ReLU$^k$ activation, the corresponding Barron space can be characterized by using the Radon transform \citep{ongie2020function,parhi2021banach,parhi2022what,mao2026approximation}, and the optimal approximation and learning rates of these functions by shallow networks have been established \citep{parhi2023minimax,siegel2025optimal,yang2024nonparametric}.

In this paper, we are mainly interested in the approximation and learning theories of spectral Barron functions. Specifically, following \citet{liao2025spectral}, we define for $s>0$,
\[
\cB^s(\bR^d) := \left\{ f\in \cS'(\bR^d) : \|f\|_{\cB^s(\bR^d)} := \int_{\bR^d} (1 + |\xi|)^s |\cF f(\xi)| d\xi<\infty \right\},
\]
where $|\xi|$ denote the Euclid norm of $\xi$ and $\cS'(\bR^d)$ is the space of tempered distributions and we use the following conversion for the Fourier transform $\cF f$ of $f\in L^1(\bR^d)$,
\[
\cF f (\xi) = \int_{\bR^d} f(x) e^{-2\pi i x \cdot \xi} dx.
\]
\citet{barron1993universal} first introduced the class with $s=1$ and showed that these functions can be approximated by shallow neural networks with sigmoidal activation without curse of dimensionality. The functional properties of $\cB^s(\bR^d)$ have been studied in the recent works \citep{meng2022new,liao2025spectral}. In particular, it was shown by \citet{liao2025spectral} that these function classes are Banach spaces and we have the embedding $B^{s+d/2}_{2,1}(\bR^d) \hookrightarrow \cB^s (\bR^d) \hookrightarrow B^{s}_{\infty,1}(\bR^d)$ between spectral Barron spaces and Besov spaces. For $\Omega=[0,1]^d$, we denote the restriction on $\Omega$ by
\[
\cB^s(\Omega) := \left\{f=g|_\Omega: g\in \cB^s(\bR^d) \right\},\quad \|f\|_{\cB^s(\Omega)} := \inf_{f=g|_\Omega} \|g\|_{\cB^s(\bR^d)},
\]
and the unit ball of this space by $\cB^s_1(\Omega) = \{f:\|f\|_{\cB^s(\Omega)}\le 1\}$. It has been shown by \citet{siegel2023characterization} that $\cB^s(\Omega)$ is equivalent to the variation space of the dictionary
\begin{equation}\label{dictionary}
\bD = \left\{ (1+|\omega|)^{-s} e^{2\pi i \omega \cdot x} : \omega\in \bR^d \right\}.
\end{equation}
Using Poisson summation formula, \citet{siegel2022high} showed that functions in $\cB^s_1(\Omega)$ can be approximated by shallow neural networks with cosine activation function and established the optimal $L^2(\Omega)$ approximation rate $\cO(W^{-1/2-s/d})$, where $W$ is the width of the network. For ReLU$^k$ shallow networks, they also gave approximation upper bound $\widetilde{\cO}(W^{-t})$ for $s\ge 1/2$, where $t=1/2 + (2s-1)/(2d+2)$ when $s<(k+1/2)(d+1)+1/2$ and $t=k+1$ otherwise. These results were extended and improved by \citet{ma2022uniform} to uniform approximation. In particular, for ReLU$^k$ activation, they derived the uniform error bound $\widetilde{\cO}(W^{-r})$, where $r=1/2 + (s-1)/(d+1)$ when $1<s\le k(d+1)+1$ and $r=k+1$ when $s\ge (k+1)(d+1)$. For deep neural networks with ReLU activation, \citet{liao2025spectral} proved $L^2(\Omega)$ error bound $\cO(W^{-sL})$ when the depth is $L$ and $0<sL\le 1/2$. Hence, their approximation bound complements previous results only when $s<1/2$. It is still unclear whether we can improve the approximation rates in \citet{siegel2022high} by using deep neural networks. For machine learning, it is also natural to ask what are the minimax rates for learning spectral Barron spaces and whether deep neural networks can achieve these minimax rates?

In this paper, we answer the above questions by providing optimal approximation bounds for deep ReLU neural networks and show that these networks can achieve the minimax rates for learning spectral Barron functions. Our main contributions are summarized as follows:

\begin{enumerate}[label=\textnormal{(\arabic*)}]
\item We establish an upper bound for the approximation of functions in the spectral Barron space $\cB^s(\Omega)$ by deep ReLU neural networks with width $W$, depth $L$ and $S$ nonzero parameters (Theorem \ref{thm:app upper bound}). Our result implies the approximation rate $\cO(S^{-\frac{1}{2}-\frac{s}{d}} (\log S)^{\frac{3}{2} + \frac{s}{d}})$ for such networks. We then derive two lower bounds $(SL)^{-\frac{1}{2}-\frac{s}{d}}$ and $(W^2 L^2)^{-\frac{1}{2}-\frac{s}{d}}$ up to logarithmic factors via metric entropy arguments in Theorems \ref{thm:app lower bound 1} and \ref{thm:app lower bound 2}. These lower bounds imply that our upper bound is optimal up to logarithmic factors in certain cases.

\item We apply our approximation result to study how well deep neural networks learn spectral Barron functions in nonparametric regression setting. By integrating our approximation bound with statistical learning theory, we show that deep ReLU networks can achieve a fast convergence rate $n^{-\frac{d+2s}{2d+2s}}$ with $n$ training samples (Theorem \ref{thm:rate}). Furthermore, we also prove that the obtained rate is minimax optimal up to logarithmic factors (Theorem \ref{thm:minimax lower bound}). To the best of our knowledge, this is the first result that characterizes the minimax rates for learning spectral Barron classes.
\end{enumerate}

The remainder of this paper is organized as follows. Section \ref{main} gives our main results on the approximation and learning theories of spectral Barron functions. Section \ref{app proof} provides the detailed proofs for the approximation upper and lower bounds. Section \ref{gen proof} presents the proofs for the generalization error bounds and the minimax lower bound. We use the following notations for asymptotic comparisons of two quantities $A$ and $B$. We denote $A \lesssim B$ (or $B \gtrsim A$ or $A=\cO(B)$) denotes the statement that $A\le CB$ for some constant $C>0$. We will also denote $A \asymp B$ when $A \lesssim B \lesssim A$ and use the notation $\widetilde{\cO}(B)$ to omit logarithmic factors.

\section{Main results} \label{main}

We consider the nonparametric regression problem, which is the one of the standard frameworks for analyzing the performances of neural networks \citep{schmidthieber2020nonparametric,nakada2020adaptive,kohler2021rate,jiao2023deep}. Suppose we have a dataset of $n\ge 2$ samples $\cD_n =\{(X_i, Y_i)\}_{i=1}^n$, which are independent and identically distributed (i.i.d.) as
\[
Y_i = f_0(X_i) + \eta_i,\quad X_i \sim \mu,\quad i=1,\dots,n,
\]
where $\mu$ is the marginal distribution of the covariate $X_i$, and the noise $\eta_i$ is independent of $X_i$ with zero mean $\bE[\eta_i]=0$. Throughout this paper, we assume that $\mu$ is supported on $\Omega =[0,1]^d$ for convenience and note that one can generalize the results to any compact sets using affine transforms. The goal of nonparametric regression is to estimate the regression function $f_0(x) = \bE[Y|X=x]$ from the observed data $\cD_n$. One of the most popular estimators is the least squares
\begin{equation}\label{least squares}
\widehat{f}_n \in \argmin_{f\in \cH_n} \frac{1}{n} \sum_{i=1}^n (f(X_i)- Y_i)^2,
\end{equation}
where $\cH_n$ is a suitably chosen hypothesis class, which may depend on the sample size $n$. For simplicity, we assume here and in the sequel that the minimum above indeed exists. We consider the case that $\cH_n$ is parameterized by a neural network defined as follows. 

Given $L,N_1,\dots, N_L \in \bN$, we consider the mapping $\phi:\bR^{d} \to \bR^{k}$ that can be parameterized by a ReLU neural network of the following form
\begin{equation}\label{NN}
\begin{aligned}
\phi_0(x) &= x, \\
\phi_{\ell+1}(x) &= \sigma(A_\ell \phi_\ell (x) + b_\ell), \quad \ell = 0,1,\dots,L-1, \\
\phi(x) &= A_L \phi_L(x) +b_L,
\end{aligned}
\end{equation}
where $A_\ell \in \bR^{N_{\ell+1}\times N_{\ell}}$, $b_\ell\in \bR^{N_{\ell+1}}$ with $N_0 =d$ and $N_{L+1} =k$. The activation function $\sigma(t) := \max\{t,0\}$ is the Rectified Linear Unit function (ReLU) and it is applied component-wisely. We remark that there is no activation function in the output layer, which is the usual convention in applications. The numbers $W:=\max\{N_1,\dots,N_L\}$ and $L$ are called the width and depth (the number of hidden layers) of the neural network, respectively. We denote by $\cN\cN_{d,k}(W,L)$ the set of mappings that can be parameterized by ReLU neural networks of the form (\ref{NN}) with width $W$ and depth $L$. In order to reduce the complexity of the function class, we will also make some restriction on the parameters in the network. Following \citet{schmidthieber2020nonparametric,nakada2020adaptive}, we denote $\NN_{d,k}(W,L,S,B)$ as the set of all $\phi \in \NN_{d,k}(W,L)$ in the form (\ref{NN}) such that
\begin{align*}
\sum_{\ell=0}^{L} (\|A_\ell\|_0 + \|b_\ell\|_0) &\le S, \\
\max_{0\le \ell\le L} \max \{\|A_\ell\|_\infty,\|b_\ell\|_\infty\} &\le B.
\end{align*}
Hence, we make sparsity constraint on the network and restrict the magnitude of the parameters. Note that we always have $S\le (d+1)W+W(W+1)(L-1)+kW+k$ and hence $S\lesssim W^2L$. When the input dimension $d$ and the output dimension $k$ are clear from contexts, we simplify the notation to $\NN(W,L,S,B)$ for convenience. 

We are going to analyze the performance of the least squares (\ref{least squares}), where the hypothesis class $\cH_n= \NN(W_n,L_n,S_n,B_n)$ is parameterized by the neural network described as above. The performance of the estimator is measured by the expected risk
\[
\cR(\widehat{f}_n) := \bE_{(X,Y)} [(\widehat{f}_n(X)-Y)^2],
\]
which is equivalent to evaluating the estimator by the excess risk
\[
\|\widehat{f}_n - f_0\|_{L^2(\mu)}^2 = \cR(\widehat{f}_n) - \cR(f_0).
\]
In order to derive convergence rate, it is necessary to make some assumption on the regression function $f_0$ that represents prior knowledge on the problem. As mentioned in the introduction, it is well-known that, for classical smoothness assumption, the minimax optimal rate suffers from the curse of dimensionality. To avoid this, we will assume that $f_0$ is in the unit ball of the spectral Barron space, i.e., $f_0\in \cB^s_1(\Omega)$. And we are going to show that neural networks can achieve the minimax rate for learning this function class. Note that we have implicitly assumed that the dataset is real-valued, so we will only consider real-valued function $f_0$ in this paper. It is easy to extend our results to the complex case if one is interested in complex-valued data.

Let $f^*\in \cH_n$ be any good approximation of $f_0$, we can decompose the error as
\[
\|\widehat{f}_n - f_0\|_{L^2(\mu)}^2 \le 2 \|f^* - f_0\|_{L^2(\mu)}^2 + 2 \|\widehat{f}_n - f^*\|_{L^2(\mu)}^2,
\]
where the first term is the approximation error of the model and the second term is estimation error due to the fact that we have only finite samples. While the estimation error can be analyzed by using statistical learning theory \citep{shalevshwartz2014understanding,mohri2018foundations}, we first bound the approximation error by explicitly constructing a good approximator $f^*$. Our approximation result is summarized in the following theorem. The proof is given in Section \ref{app proof}.

\begin{theorem}\label{thm:app upper bound}
For any $f_0\in\cB_1^s(\Omega)$ with $s>0$ and any $\varepsilon\in (0,1)$, there exists a ReLU network $f^*\in\NN(W,L,S,B)$ with
\[
W \lesssim \varepsilon^{-\frac{2d}{d+2s}}\left(\log\frac{1}{\varepsilon}\right)^{\frac{d}{d+2s}},\quad L \lesssim \left(\log \frac{1}{\varepsilon} \right)^2,\quad S \lesssim \varepsilon^{-\frac{2d}{d+2s}}\left(\log\frac{1}{\varepsilon}\right)^{\frac{3d+4s}{d+2s}},\quad B \lesssim 1,
\]
such that
\[
\|f^*-f_0\|_{L^\infty(\Omega)} \le \varepsilon.
\]
\end{theorem}

Representing $\varepsilon$ by $S$ in Theorem \ref{thm:app upper bound}, we can characterize the approximation error by the number of parameters in the network as
\begin{equation}\label{app rate}
\|f^*-f_0 \|_{L^{\infty}(\Omega)} \lesssim S^{-\frac{1}{2}-\frac{s}{d}}
(\log S)^{\frac{3}{2}+\frac{2s}{d}},
\end{equation}
for some network with width depth $L\lesssim (\log S)^2$, network size $WL \asymp S$ and bounded weights. This bound improves the recent results of \citet{liao2025spectral}, who established the $L^2$ approximation bound $\cO(W^{-sL})$ for $0<sL<1/2$, since their approximation rate is slower than $\cO(S^{-1/2})$. Note that they also proved their approximation upper bound is not improvable when $0<sL<1/2$. Thus, Theorem \ref{thm:app upper bound} implies the advantage of very deep neural networks for approximating spectral Barron functions. On the other hand, \citet{siegel2022high,ma2022uniform} analyzed the approximation capacities of shallow networks. In particular, \citet{ma2022uniform} gave the uniform upper bound $\cO(W^{-\frac12-\frac{s}{d}} \sqrt{\log W})$ for cosine activation and the bound $\widetilde{\cO}(W^{-r})$ for ReLU activation, where $r=1/2 + (s-1)/(d+1)$ when $1<s\le d+2$ and $r=2$ when $s\ge 2d+2$. Since $S\asymp W$ for shallow networks, our result show that deep ReLU networks can achieve the same approximation rate as shallow cosine networks up to logarithmic factors and improve their approximation rate for ReLU networks in terms of the number of parameters.

In the following theorem, we give a lower bound for the $L^2$ approximation error, which indicates that the rate in (\ref{app rate}) is optimal in certain situations.

\begin{theorem}\label{thm:app lower bound 1}
For any $s>0$, there exists $f_0 \in \cB_1^s(\Omega)$ such that, for any $B\ge 1$,
$$
\inf_{f \in \NN(W,L,S,B)} \|f - f_0\|_{L^2(\Omega)} \gtrsim (SL \log (SB))^{-\frac{1}{2}-\frac{s}{d}}.
$$ 
\end{theorem}

Note that the lower bound in Theorem \ref{thm:app lower bound 1} is different from the upper bound (\ref{app rate}), which does not depend on the depth $L$. However, in our construction, the depth $L\lesssim (\log S)^2$ and $B\lesssim 1$ for the constructed network that satisfies the upper bound (\ref{app rate}). For these networks, Theorem \ref{thm:app lower bound 1} gives the lower bound $(S(\log S)^3)^{-\frac{1}{2}-\frac{s}{d}}$, which implies that (\ref{app rate}) is optimal up to logarithmic factors.

We remark that we assume the weights in the network are uniformly bounded by $B$ in Theorem \ref{thm:app lower bound 1}. When this restriction is removed, we can derive another lower bound that characterizes the approximation error only using width and depth.

\begin{theorem}\label{thm:app lower bound 2}
For any $s>0$, there exists $f_0 \in \cB_1^s(\Omega)$ such that
$$
\inf_{f \in \NN(W,L)} \|f_0 - f\|_{L^2(\Omega)} \gtrsim \min\left\{W^2 L^2 \log(W L), W^3 L^2\right\}^{^{-\frac{1}{2}-\frac{s}{d}}} (\log (WL))^{^{-\frac{1}{2}-\frac{s}{d}}}.
$$ 
\end{theorem}

For fully-connected networks with depth $L\ge 2$, the number of parameters in the network is $S\asymp W^2L$. Hence, the lower bound in Theorem \ref{thm:app lower bound 2} is of a similar form as Theorem \ref{thm:app lower bound 1}. Furthermore, when the width $W$ is bounded, we have $S\asymp L$ and the approximation error is lower bounded by $(S^2 \log S)^{-1/2-s/d}$, which has a fast convergence rate than the upper bound (\ref{app rate}). It implies that it might be possible to improve the approximation rate by using very deep networks. Similar phenomenon, which is called supper approximation, has been widely studied for the approximation of classical smooth functions \citep{yarotsky2018optimal,shen2020deep,lu2021deep,siegel2023optimal,yang2025optimal}. It is an interesting problem to determine whether the supper approximation rate also holds for spectral Barron functions and we leave this for future study.

Now, let us come back to our nonparametric regression problem. In order to derive high probability bound for the convergence of the least square (\ref{least squares}) with deep neural network models, we make the following assumption on the noises: there exits constants $c_\eta \ge 1$ and $q>0$ such that
\begin{equation}\label{noise assumption}
\bP\left(|\eta_i|\le c_\eta t^q\right) \ge 1-e^t, \quad \forall t>1.
\end{equation}
Note that $q=1/2$ corresponds to sub-Gaussian noise and $q=1$ is sub-exponential \citep[Theorems 2.6 and 2.13]{wainwright2019high}, while bounded noises can be viewed as the case $q=0$. Thus, this is slightly more general than the usual sub-Gaussian assumption in the literature. For statistical analysis of learning algorithms, we often require that the hypothesis class is uniformly bounded. To achieve this, we define the truncation operator $\cT_B$ with level $B>0$ for real-valued functions $h$ as
\begin{equation}\label{trancation}
\cT_Bh(x) := 
\begin{cases}
h(x) &\quad \mbox{if }|h(x)|\le B, \\
\sgn(h(x)) B &\quad \mbox{if } |h(x)|> B.
\end{cases}
\end{equation}
For a function class $\cH$ containing real-valued functions, we denote $\cT_B \cH := \{\cT_Bh: h\in \cH\}$ for convenience. Observing that for $f\in \cB^s(\Omega)$, if $g\in \cB^s(\bR^d)$ is any extension of $f$, then we have
\[
|f(x)| \le \int_{\bR^d} |\cF g(\xi)| d\xi \le \int_{\bR^d} (1 + |\xi|)^s |\cF g(\xi)| d\xi = \|g\|_{\cB^s(\bR^d)},
\]
which implies $\|f\|_{L^\infty(\Omega)} \le \|f\|_{\cB^s(\Omega)}$. Thus, if the regression function $f_0\in \cB^s_1(\Omega)$, we can truncate the estimator $\widehat{f}_n$ by $\cT_1$, which can only reduce the generalization error since $\|\cT_1 \widehat{f}_n - f_0\|_{L^2(\mu)} \le \|\widehat{f}_n - f_0\|_{L^2(\mu)}$. We give convergence rate for the truncated estimator in the following theorem. The proof is deferred to Section \ref{rate proof}.

\begin{theorem}\label{thm:rate}
Assume that the noises satisfy (\ref{noise assumption}) and the regression function $f_0 \in \cB_1^s(\Omega)$ for some $s>0$. Let $\widehat{f}_n$ be the least squares estimator (\ref{least squares}) with hypothesis class $\cH_n=\NN(W_n,L_n,S_n,B_n)$. If we choose
\begin{align*}
W_n &\asymp n^{\frac{d}{2d+2s}} (\log n)^{-\frac{5d}{2d+2s}}, && L_n \asymp (\log n)^2, \\
S_n &\asymp  n^{\frac{d}{2d+2s}} (\log n)^{\frac{4s-d}{2d+2s}}, && B_n \asymp 1,
\end{align*}
then there exist constants $c_1,c_2,C>0$ such that for any $\rho \ge 1$,
\[
\|\cT_1\widehat{f}_n - f_0\|_{L^2(\mu)}^2 \le C(\rho + \log n)^{2q} n^{-\frac{d+2s}{2d+2s}} (\log n)^{\frac{7d+12s}{2d+2s}},
\]
holds with probability at least $1 - \exp(-\rho) - c_1 \exp(-c_2 n^{\frac{d}{2d+2s}}(\log n)^{\frac{7d+12s}{2d+2s}})$ over the dataset $\cD_n$.
\end{theorem}

While the approximation theory of spectral Barron functions by neural networks have been studied in several recent works \citep{klusowski2018approximation,siegel2022high,ma2022uniform,liao2025spectral}, we do not aware any generalization error bounds for learning these functions in the literature. It is well-known that the minimax rate for leaning Sobolev and Besov functions with smoothness $s$ is $n^{-\frac{2s}{2s+d}}$, which suffers from the curse of dimensionality. In contrast, Theorem \ref{thm:rate} shows that leaning spectral Barron space $\cB^s(\Omega)$ converges in the rate $\widetilde{\cO}(n^{-\frac{d+2s}{2d+2s}})$, which is always faster than $n^{-1/2}$ and approaches the rate $n^{-1}$ when $s\to \infty$. This result shows that deep neural networks are able to adaptive to the non-classical smoothness defined through spectral Barron functions, even in high dimension.

Finally, we show that the rate in Theorem \ref{thm:rate} is minimax optimal up to logarithmic factors. The proof is given in Section \ref{optimality}.

\begin{theorem}\label{thm:minimax lower bound}
Assume that $\mu$ is the uniform distribution on $\Omega$ and the noise $\eta_i \sim \cN(0,\sigma^2)$ is Gaussian with variance $\sigma^2>0$. For any $s>0$, there exists a constant $c>0$ depending only on $s,d$ and $\sigma^2$ such that
\[
\inf_{\widehat{f}} \sup_{f_0 \in \cB_1^s(\Omega)} \bE_{\cD_n}\left[ \|\widehat{f} - f_0\|_{L^2(\Omega)}^2 \right] \ge c \left(n(\log n)^{1+\frac{2s}{d}}\right)^{-\frac{d+2s}{2d+2s}},
\]
where the infimum is taken over all measurable estimators based on the dataset $\cD_n$.
\end{theorem}

\section{Approximation error} \label{app proof}

In this section, we first prove Theorem \ref{thm:app upper bound} by explicitly constructing a ReLU neural network with small approximation error. We then complement this result by establishing the lower bounds in Theorems \ref{thm:app lower bound 1} and \ref{thm:app lower bound 2}.

\subsection{Upper bounds}

Let us begin with a review on the results of \citet{siegel2022high} and \citet{ma2022uniform}. By using Poisson summation formula, \citet[Corollary 1]{siegel2022high} showed that, if $f_0\in \cB_1^s(\Omega)$, then for any $L_0>1$, there exists an $\alpha \in L_0^{-1}[0,1]^d$ (potentially depending upon $f_0$) such that $f_0$ can be represented as an infinite series:
\[
f_0(x)=\sum_{\omega\in \alpha+L_0^{-1}\mathbb Z^d} c_\omega(1+|\omega|)^{-s}e^{2\pi i\omega\cdot x},
\]
where the coefficients $c_\omega$ satisfy
\[
\sum\limits_{\omega \in  \alpha+L_0^{-1} \mathbb{Z}^d} |c_{\omega}| \lesssim \|f\|_{\cB^s(\Omega)} \le 1.
\]
Using this representation, \citet[Theorem 2]{ma2022uniform} showed that $f_0$ can be approximated by finite linear combinations of elements from the dictionary (\ref{dictionary}). Specifically, for any integer $N\ge 2$, there exists
\[
f_N(x) = \sum_{q=1}^{N} C_q (1+|\omega_q|)^{-s} e^{2\pi i\omega_q\cdot x}\quad  \mbox{with} \quad \sum_{q=1}^{N} |C_q|\lesssim 1,
\]
such that
\begin{equation}\label{temp1}
\|f_0 - f_N\|_{L^\infty(\Omega)} \lesssim N^{-\frac12-\frac{s}{d}}\sqrt{\log N}.
\end{equation}

Recall that we have assumed $f_0$ is real-valued, hence we can simply approximate $f_0$ by the real part of $f_N$. If we write $C_q=A_q-iB_q$ with $A_q,B_q\in\mathbb R$, then, using Euler's formula,
\[
\operatorname{Re}\left(C_q e^{2\pi i\omega_q\cdot x}\right)= A_q\cos(2\pi\omega_q\cdot x)+B_q\sin(2\pi\omega_q\cdot x).
\]
Thus there exist real coefficients $A_q,B_q$ with
$|A_q|,|B_q|\le |C_q|$ such that
\[
\operatorname{Re} f_N(x)=\sum_{q=1}^{N} (1+|\omega_q|)^{-s} \left(A_q\cos(2\pi\omega_q\cdot x)+B_q\sin(2\pi\omega_q\cdot x)\right),
\quad
\sum_{q=1}^{N}\left(|A_q|+|B_q|\right)\lesssim 1.
\]
In our construction, we will need to truncate the frequencies $|\omega_q|\le T$ by some appropriately chosen $T>0$. Thus, we define
\begin{equation}\label{f_NT}
f_{N,T} = \sum_{|\omega_q|\le T, 1\le q\le N} (1+|\omega_q|)^{-s} \left(A_q\cos(2\pi\omega_q\cdot x)+B_q\sin(2\pi\omega_q\cdot x)\right).
\end{equation}
Without loss generality, we may assume that the above summation is over $1\le q\le N^*$ for some $N^*\le N$. For this truncation, the approximation error is
\begin{align*}
\| \operatorname{Re} f_N - f_{N,T} \|_{L^\infty(\Omega)} &\le \sum_{|\omega_q|> T, 1\le q\le N} (1+|\omega_q|)^{-s}(|A_q|+|B_q|) \\
&\le (1+T)^{-s} \sum_{q=1}^{N} (|A_q|+|B_q|) \\
&\lesssim (1+T)^{-s}.
\end{align*}
Let us denote the error in (\ref{temp1}) by $\cE_N= N^{-\frac12-\frac{s}{d}}\sqrt{\log N}$ and choose $T=\cE_N^{-1/s}$. Then,
\begin{equation}\label{temp2}
\begin{aligned}
\| f_0 - f_{N,T} \|_{L^\infty(\Omega)} &\le \|f_0 - \operatorname{Re} f_N\|_{L^\infty(\Omega)} + \| \operatorname{Re} f_N - f_{N,T} \|_{L^\infty(\Omega)} \\
&\lesssim \cE_N + (1+T)^{-s} \\
&\lesssim \cE_N.
\end{aligned}
\end{equation}

\begin{remark}\label{app remark}
If we measure the approximation error in $L^2(\Omega)$ norm, we can get a slightly better bound $\|f_0 - f_N\|_{L^2(\Omega)} \lesssim N^{-\frac12-\frac{s}{d}}$ by \citep[Theorem 1]{siegel2022high}. Using this bound in the following analysis will give some slight improvements for our approximation results, because we can choose $\cE_N \asymp N^{-\frac12-\frac{s}{d}}$ in (\ref{temp2}). However, in nonparametric regression, since the marginal distribution $\mu$ is unknown, we can only bound the $L^2(\mu)$ approximation error by the uniform error rather than the $L^2(\Omega)$ error, unless we assume that $\mu$ is absolutely continuous with respect to the Lebesgue measure with bounded density.
\end{remark}

We have reduced the problem of approximating $f_0$ by neural networks to the approximation of $f_{N,T}$. To approximate $f_{N,T}$, it is necessary to approximate the cosine function by ReLU networks. This can be done by the following lemma, which is adapted from \citet[Theorem 4.1]{perekrestenko2018universal}

\begin{lemma}\label{lem:cos}
There exist absolute constants $c_1,c_2,c_3>0$ such that for any $\varepsilon\in (0,1)$ and any $D\ge 1$, there exists a neural network $f\in \NN(W,L,S,B)$ with $W = 21$, $L \leq c_1 (\log (1/\varepsilon))^2 + c_2 \log D $,  $S \le 462L$,  $B \le c_3$ such that
	$$
	\|f - \cos(\cdot )\|_{L^\infty([-D,D])} \le \varepsilon.
	$$
\end{lemma}
\begin{proof}
\citet[Theorem 4.1]{perekrestenko2018universal} constructed the desired network for the given $W,L$ and $B$. The main idea is to approximate the cosine function by its Taylor expansion and use the fact that polynomials can be efficiently approximated by ReLU networks as shown by \citet{yarotsky2017error}. Since the width $W=21$ is fixed, the number of parameters in the network is at most $S \le 2W+W(W+1)(L-1)+W+1 \le 462L$.
\end{proof}

Using Lemma \ref{lem:cos}, we can construct a neural network to approximate the function $F_{N,T}$ and hence give  an approximation of the target function $f_0$. The construction is given in the following proof.

\begin{proof}[Proof of Theorem \ref{thm:app upper bound}]
Recall that $f_{N,T}$ is defined by (\ref{f_NT}). For each $q$, we denote the linear function
$$
t_q(x):=2\pi\omega_q\cdot x = \sum_{j=1}^d 2\pi \omega_{q,j}x_j,
$$
for $x=(x_1,\dots,x_d) \in \Omega =[0,1]^d$ and $\omega_q=(\omega_{q,1},\dots,\omega_{q,d})$. We note that a direct computation of the dot product using linear layer would result in parameter weights scaling with $T$, leading to an undesirable dependence on the frequency magnitude. To avoid this, we observe that
\[
\sigma_\ell(t) = \underbrace{\sigma(2 \sigma(2 \cdots \sigma(2t)))}_{\ell \mbox{ layers}} = 2^\ell t, \quad t\ge 0.
\]
Hence, we can compute $t_q(x)$ using these networks as 
\[
t_q(x) = \sum_{j=1}^d 2\pi \omega_{q,j}x_j = \sum_{j=1}^d \frac{2\pi \omega_{q,j}}{2^\ell} \sigma_\ell(x_j),\quad x\in \Omega.
\]
In other words, $t_q$ is a linear combination of $d$ subnetworks $\sigma_\ell$. By choosing $\ell = \lceil \log_2(\pi T)\rceil$, we are guaranteed that all coefficients $2\pi \omega_{q,j}/2^{\ell} \le 2$. It is easy to see that $\sigma_\ell \in \NN(1,\ell,\ell+1,2)$. Placing these $d$ subnetworks in parallel and computing the linear combination in the output layer show that $t_q\in \NN(W_{\rm lin},L_{\rm lin},S_{\rm lin},B_{\rm lin})$
with width $W_{\rm lin}=d$, depth $L_{\rm lin} = d\ell = d \lceil \log_2(\pi T)\rceil \lesssim \log T$, sparsity $S_{\rm lin} = d(\ell +1) \lesssim \log T$, and weight magnitude $B_{\rm lin}=2$. 

Since $x\in \Omega$ and $|\omega_q|\le T$, we have $|t_q(x)|\le 2\pi dT=:D_0$. Let $\cE_N>0$ be sufficiently small and denote $D_0:=D+\pi/2$. By Lemma \ref{lem:cos}, there exist constants $c_1,c_2,c_3>0$ and a neural network $\Phi_{\cos}\in\NN(W_{\cos},L_{\cos},S_{\cos},B_{\cos})$ with
$W_{\cos}=21$, $L_{\cos}\le c_1 (\log (1/\cE_N))^2+c_2\log D_0$, $S_{\cos}\le 462L_{\cos}$, $B_{\cos}\le c_3$,
such that
$$
\left\|\Phi_{\cos}(\cdot)-\cos (\cdot)\right\|_{L^\infty([-D_0,D_0])}
\le \cE_N.
$$
For the sine function, we use the identity $\sin\theta=\cos(\theta-\pi/2)$ and define
$$
\Phi_{\sin}(t):=\Phi_{\cos}(t-\pi/2).
$$
Since $t\in[-D,D]$ implies $t-\pi/2\in[-D_0,D_0]$, we obtain
$$
\left\|\Phi_{\sin}(\cdot)-\sin (\cdot)\right\|_{L^\infty([-D,D])}
=
\left\|\Phi_{\cos}(\cdot-\pi/2)-\cos(\cdot-\pi/2)\right\|_{L^\infty([-D,D])}
\le \cE_N.
$$
Thus the sine subnetwork shares the same asymptotic complexity as the cosine subnetwork.  Consequently, we may use the same parameter notation $L_{\cos}$, $W_{\cos}$, $S_{\cos}$, $B_{\cos}$ for both the cosine and sine subnetworks without loss of generality.

We now combine the linear layer and the trigonometric subnetworks to define our approximator
$$
f^*(x):=\sum_{q=1}^{N^*} (1+|\omega_q|)^{-s} \left(A_q\Phi_{\cos}(t_q(x))+B_q\Phi_{\sin}(t_q(x))\right),
$$
where we have assumed that $\{1\le q\le N:|\omega_q|\le T\} = \{1,\dots,N^*\}$ in (\ref{f_NT}). The approximation error of $f^*$ to $f_{N,T}$ can be bounded as
\begin{align*}
&\|f_{N,T}-f^*\|_{L^\infty(\Omega)} \\
\le&\ \sum_{q=1}^{N^*} \left(|A_q| \max_{x\in \Omega} \left|\cos(2\pi\omega_q\cdot x) - \Phi_{\cos}(t_q(x))\right| +|B_q| \max_{x\in \Omega} \left|\sin(2\pi\omega_q\cdot x) - \Phi_{\sin}(t_q(x))\right| \right) \\
\le&\ \sum_{q=1}^{N^*}\left(|A_q|+|B_q|\right)\cE_N
\lesssim \cE_N.
\end{align*}
Combining this bound with (\ref{temp2}), we get
$$
\|f_0-f^*\|_{L^\infty(\Omega)}
\le \|f_0-f_{N,T}\|_{L^\infty(\Omega)}+\|f_{N,T}-f^*\|_{L^\infty(\Omega)}
\lesssim \cE_N.
$$
Given any small $\varepsilon>0$, we can choose $\cE_N \lesssim \varepsilon$ such that the above approximation error is bounded by $\varepsilon$. Recall that $\cE_N= N^{-\frac12-\frac{s}{d}}\sqrt{\log N}$ and $T=\cE_N^{-1/s}$. Thus, we can choose
\[
N \asymp \varepsilon^{-\frac{2d}{d+2s}} \left(\log (1/\varepsilon) \right)^{\frac{d}{d+2s}},\quad T\asymp \varepsilon^{-1/s}.
\]

It remains to estimate the size of the network $f^*\in \NN(W,L,S,B)$. Since we can compute $2N^*$ subnetworks $\Phi_{\cos}(t_q(x))$ and $\Phi_{\sin}(t_q(x))$ in parallel and use the output layer to compute the linear combination. The width $W$ is the sum of the widths of these $2N^*$ subnetworks, while the width of each subnetwork is bounded by the maximum width of $t_q$ and $\Phi_{\cos}$. Thus,
\[
W = 2N^* \max\{W_{\rm lin}, W_{\cos} \} \lesssim N \lesssim \varepsilon^{-\frac{2d}{d+2s}} \left(\log (1/\varepsilon) \right)^{\frac{d}{d+2s}}.
\]
The depth $L$ can be bounded by maximum depth among all parallel subnetworks, while the depth of each subnetwork is bounded by the sum of depths of $t_q$ and $\Phi_{\cos}$. Consequently,
\[
L = L_{\rm lin} + L_{\cos} \lesssim \log T + (\log (1/\cE_N))^2 + \log D_0 \lesssim (\log (1/\varepsilon))^2.
\]
The sparsity $S$ of $f^*$ can be bounded by the sum of the parameters in all $2N^*$ subnetworks:
\begin{align*}
S &\lesssim 2 N^* (S_{\rm lin} + S_{\cos}) \lesssim N(\log T + (\log (1/\cE_N))^2 + \log D_0 ) \\
&\lesssim \varepsilon^{-\frac{2d}{d+2s}} \left(\log (1/\varepsilon) \right)^{\frac{3d+4s}{d+2s}}.
\end{align*}
Finally, the maximum magnitude is
\[
B \le \max \left\{ B_{\rm lin}, B_{\cos}, (1+|\omega_q|)^{-s}|A_q|, (1+|\omega_q|)^{-s}|B_q| \right\} \lesssim 1,
\]
since we have $|A_q|, |B_q|\lesssim 1$.
\end{proof}

\subsection{Lower bounds}
In this section we establish two lower bounds for the approximation error, corresponding to two different network architecture. The first bound is for networks with sparsity and weight constraints. The second lower bound can be applied for any ReLU networks with bounded width and depth. Our lower bounds are derived via metric entropy arguments. We begin by recalling the definitions of covering and packing numbers.

\begin{definition}(Metric entropy). \label{def:cov}
Let $\rho$ be the metric on a set $\mathcal{X}$ and $S \subset \mathcal{X}$. For $\varepsilon > 0$, a set $T \subset \mathcal{X}$ is called an $\varepsilon$-cover of $S$ if for any $x \in S$ there exists $y \in T$ such that $\rho(x,y) \leq \varepsilon$. A subset $U \subset S$ is called an $\varepsilon$-packing (or $\delta$-separated subset) of $S$ if any two distinct elements $x, y \in U$ satisfy $\rho(x,y) > \varepsilon$. The $\varepsilon$-covering and $\varepsilon$-packing numbers of $S$ are defined respectively by
\begin{align*}
\mathcal{N}(\varepsilon, S, \rho) &:= \min\{ |T| : T \text{ is an } \varepsilon\text{-cover of } S \}, \\
\mathcal{M}(\varepsilon, S, \rho) &:= \max\{ |U| : U \text{ is an } \varepsilon\text{-packing of } S \}.
\end{align*}
The quantities $\log \mathcal{N}(\varepsilon, S, \rho)$ and $\log \mathcal{M}(\varepsilon, S, \rho)$ are called metric entropy of $S$.
\end{definition}
It is well known that $\mathcal{M}(2\varepsilon,S,\rho) \leq \mathcal{N}(\varepsilon,S,\rho) \leq \mathcal{M}(\varepsilon,S,\rho)$. In many cases, the metric $\rho$ is induced by a norm $\|\cdot\|$ and we denote $\cN(\varepsilon, S, \|\cdot\|)$ and $\cM(\varepsilon, S, \|\cdot\|)$ for convenience.

We derive approximation lower bound by comparing some upper bounds on the packing number of the neural network class and lower bound on the packing number of spectral Barron space. \citet[Lemma 21]{nakada2020adaptive} showed that the metric entropy of $\NN(W,L,S,B)$ can be upper bounded as
\begin{equation}\label{nn up covering}
\log \cN\left(\varepsilon,\NN(W,L,S,B),\|\cdot\|_{L^{\infty}(\Omega)}\right) 
\leq S \log\left(\frac{2LB^L(S+1)^L}{\varepsilon}\right).
\end{equation}
Since $\|f\|_{L^2(\Omega)} \leq \|f\|_{L^\infty(\Omega)}$ for any $f$, any covering in the $L^\infty$ norm is also a covering in the $L^2$ norm. Combining this with the inequality $\mathcal{M}(2\varepsilon,S,\rho) \leq \mathcal{N}(\varepsilon,S,\rho)$, we immediately obtain 
\begin{equation}\label{nn up packing}
\log \mathcal{M}\left(\varepsilon,\NN(W,L,S,B),\|\cdot\|_{L^{2}(\Omega)}\right) 
\leq S \log\left(\frac{4LB^L(S+1)^L}{\varepsilon}\right) .
\end{equation}

Next, we construct a family of functions   that are well separated so that they are difficult to be approximated by neural networks. The construction is a modification of the method in \citet[Proof of Theorem 6]{siegel2022high}. There the authors analyzed the dictionaries of complex exponential functions. Here, we instead consider real-valued spectral Barron spaces generated by dictionaries of cosine functions.

\begin{lemma}\label{lem:barron-separated-set}
	Let $s>0$ and $\Omega=[0,1]^d$.
	For every \(R\ge 1\), there exist $N_R\ge 2^{cR^d}$	functions $f_1,\dots,f_{N_R}\in \mathbb B_1^s(\Omega)$
	such that
	$$
	\|f_i-f_j\|_{L^2(\Omega)} > c_0 R^{-s-d/2},
	\qquad 
	\forall i\ne j,
	$$
	where $c,c_0>0$ are constants depending only on $d$ and $s$.
\end{lemma}
\begin{proof}
It is enough to prove the lower bound for large \(R\). We choose the frequency set  
$$
\cS_R:=\left\{\omega=(\omega_1,\dots,\omega_d)\in\bZ^d:\|\omega\|_\infty\le R,\omega_1> 0\right\}.
$$
It is easy to see that its cardinality $M_R=|\cS_R|\asymp R^d$. We define the dictionary 
$$
\bD_R:=\left\{\phi_\omega(x):=(1+|\omega|)^{-s}\cos(2\pi \omega \cdot x):\omega\in \cS_R\right\}.
$$
For $\omega\neq 0$, the Fourier transform of $\phi_\omega$ is  
$$
{\mathcal{F}\phi_\omega}(\xi)=\frac{1}{2} (1+|\omega|)^{-s} \left(\delta_\omega(\xi)+\delta_{-\omega}(\xi)\right),
$$
where $\delta_\omega$ denotes the Dirac delta function at $\omega$. Therefore,
$$
\|\phi_\omega\|_{\cB^s(\Omega)}	\le	\int_{\mathbb R^d}(1+|\xi|)^s|{\mathcal{F}\phi_\omega}(\xi)|d\xi
=\frac12(1+|\omega|)^{-s}\left((1+|\omega|)^s+(1+|\omega|)^s\right) =1.
$$
Hence, the dictionary $\bD_R$ is a subset of $\cB^s_1(\Omega)$. Moreover, for any $\phi_\omega \in \bD_R$, we have
\[
\|\phi_\omega\|_{L^2(\Omega)}^2 = (1+|\omega|)^{-2s} \int_\Omega \cos(2\pi \omega \cdot x)^2 dx = \frac{1}{2} (1+|\omega|)^{-2s} \gtrsim R^{-2s}
\]

Next, we show the $L^2(\Omega)$-orthogonality of the dictionary. For $\omega,\omega'\in\mathcal \cS_R$ with $\omega\neq\omega'$, the product-to-sum formula gives  
$$
\cos(2\pi\omega\cdot x)\cos(2\pi\omega'\cdot x)=
\frac12\left(\cos(2\pi(\omega+\omega')\cdot x)+\cos(2\pi(\omega-\omega')\cdot x)\right).
$$
Since $\omega\neq\omega'$, we have $\omega-\omega'\neq0$. Moreover, by the choice of $\cS_R$, we cannot have $\omega+\omega'=0$, because $\omega_1,\omega_1' >0$. Hence, both $\omega+\omega'$ and $\omega-\omega'$ are nonzero vectors in $\mathbb Z^d$, so that the integrals of $\cos(2\pi(\omega+\omega')\cdot x)$ and $\cos(2\pi(\omega-\omega')\cdot x)$ over $[0,1]^d$ vanish. Consequently,
$$
\int_\Omega \phi_\omega(x)\phi_{\omega'}(x) dx = (1+|\omega|)^{-s}(1+|\omega'|)^{-s} \int_\Omega \cos(2\pi\omega\cdot x)\cos(2\pi\omega'\cdot x)dx=0.
$$
In other words, any two elements in $\bD_R$ are orthogonal in $L^2(\Omega)$. 

For any function $a:\cS_R \to \{0,1\}$, we can define
\[
f_a = \frac{1}{M_R} \sum_{\omega\in \cS_R} a(\omega) \phi_\omega.
\]
It is easy to see that $f_a\in \cB^s_1(\Omega)$ since $\|\phi_\omega\|_{\cB^s(\Omega)} \le 1$. Note that we can view the function $a$ as a binary sequence of length $M_R=|\cS_R|$. Applying Gilbert-Varshamov bound \citep[Lemma 2.9]{tsybakov2009introduction} to the set of all these functions $a$, we can get a subset $\cA$ with cardinality $N_R=|\cA| \ge 2^{M_R/8}\ge 2^{cR^d}$ such that the Hamming distance between any two functions $a,a' \in \cA$ can be lower bounded as
\[
\left|\{\omega\in \cS_R: a(\omega)\neq a'(\omega)\}\right| \ge \frac{M_R}{8}.
\]
The $L^2(\Omega)$ distance between the corresponding functions $f_a$ and $f_{a'}$ is
\begin{align*}
\|f_a-f_{a'}\|_{L^2(\Omega)}^2  &= \frac{1}{M_R^2} \left\|\sum_{\omega \in \cS_R} (a(\omega) - a'(\omega)) \phi_\omega \right\|_{L^2(\Omega)}^2 \\
&= \frac{1}{M_R^2} \sum_{\omega: a(\omega)\neq a'(\omega)} \|\phi_\omega\|_{L^2(\Omega)}^2 \\
&\gtrsim M_R^{-1} R^{-2s} \gtrsim R^{-d-2s},
\end{align*}
where we use the $L^2(\Omega)$-orthogonality of the dictionary in the second equality.
\end{proof}

Now we are ready to prove the lower bound for the approximation error of the neural network class $\NN(W,L,S,B)$, under the assumption that the weight magnitude $B$ are bounded. The proof proceeds as follows. We have constructed a large finite set of functions in $\cB_1^s(\Omega)$ that are pairwise well separated in $L^2(\Omega)$. This gives a lower bound on the packing number of $\cB_1^s(\Omega)$. If we can approximate $\cB_1^s(\Omega)$ by the neural network class $\NN(W,L,S,B)$, then we can compare this lower bound with the upper bound (\ref{nn up packing}). If the approximation error were too small, these two bounds for the packing numbers would lead to a contradiction.  Hence the error cannot be arbitrarily small, which yields the claimed lower bound.

\begin{proof}[Proof of Theorem \ref{thm:app lower bound 1}]
Let $\varepsilon>0$ be arbitrarily small and denote the approximation error by
\[
\delta = \sup_{f\in \cB^s_1(\Omega)} \inf_{g\in \NN(W,L,S,B)} \|f-g\|_{L^2(\Omega)} + \varepsilon.
\]
Let $f_1,\dots,f_{N_R}\in \mathbb B_1^s(\Omega)$ be the functions in Lemma \ref{lem:barron-separated-set} so that $N_R\ge 2^{cR^d}$ and
$$
\|f_i-f_j\|_{L^2(\Omega)} > c_0 R^{-s-d/2} =: \varepsilon_0,
\qquad 
\forall i\ne j.
$$
By the definition of approximation error, for each $i=1,\dots,N_R$, there exists a neural network $g_i \in \NN(W,L,S,B)$ such that $\|f_i - g_i\|_{L^2(\Omega)} \le \delta$.

We are going to show that $\delta\ge \varepsilon_0/3$ for some $R$ depending on $S$ and $L$. Suppose on the contrary that $\delta< \varepsilon_0/3$. Then, for $i\ne j$,
$$
\begin{aligned}
	\| g_i -  g_j\|_{L^2(\Omega)}
	&\ge \|f_i - f_j\|_{L^2(\Omega)}
	- \|f_i -  g_i\|_{L^2(\Omega)}
	- \|f_j -  g_j\|_{L^2(\Omega)} \\
	&> \varepsilon_0 - 2\delta > \varepsilon_0/3.
\end{aligned}
$$
Hence, $\{g_i\}_{i=1}^{N_R}$ in the network class $\NN(W,L,S,B)$ are mutually $\varepsilon_0/3$-separated in $L^2(\Omega)$. Therefore, the packing number of the network class satisfies
\begin{equation}
\mathcal{M}\left(\varepsilon_0/3, \NN(W,L,S,B), \|\cdot\|_{L^2(\Omega)}\right) \ge N_R \ge 2^{cR^d}. 
\label{app low rd 1}
\end{equation}
Consequently, combining (\ref{app low rd 1}) with (\ref{nn up packing}), we have
$$
R^d 
\lesssim\log \mathcal{M}\left(\varepsilon_0/3, \NN(W,L,S,B), \|\cdot\|_{L^{\infty}(\Omega)}\right) 
\lesssim S \log\left(\frac{12LB^L(S+1)^L}{\varepsilon_0}\right).
$$
Since $L\le S$ and $\varepsilon_0 = c_0 R^{-s-d/2}$, the above inequality implies
\begin{equation}\label{lower b1}
R \lesssim (S L\log (SBR))^{1/d}.
\end{equation}
It is possible to choose $R = C (S L\log (SB))^{1/d}$ for some constant $C>0$ such that the inequality (\ref{lower b1}) is always false. This contradiction implies $\delta\ge \varepsilon_n/3$. With this choice of $R$,
\[
\delta \ge \varepsilon_0/3 \gtrsim R^{-s-\frac{d}{2}} \gtrsim (SL \log (SB))^{-\frac{1}{2}-\frac{s}{d}},
\]
which completes the proof by the definition of $\delta$.
\end{proof}

Next, we derive another lower bound for the approximation error of neural network class $\NN(W,L)$, where there is no restriction on the weights. To do this, we need to measure the complexity of the function class in another way. We begin with the definition of pseudo-dimension, which is widely used in machine learning.

\begin{definition}(Pseudo-dimension). 
	Let $\mathcal{H}$ be a class of real-valued functions defined on a set $\mathcal{X}$. A set $S=\{x_1,\ldots,x_m\}\subseteq\mathcal{X}$ is shattered by $\mathcal{H}$ if there exist constants $t_1,\ldots,t_N\in\bR$ such that
	$$
	\left|\left\{\left(\operatorname{sgn}(h(x_1)-t_1),\ldots,\operatorname{sgn}(h(x_N)-t_N)\right): h\in\mathcal{H}\right\}\right| = 2^N.
	$$
	The pseudo-dimension $\Pdim(\mathcal{H})$ is the maximum cardinality of a set $S\subseteq\mathcal{X}$ that is shattered by $\mathcal{H}$.
\end{definition}

The pseudo-dimension of neural networks have been extensively studied in machine learning theory \citep{anthony2009neural}. The recent work of \citet{bartlett2019nearly} gave nearly tight pseudo-dimension bounds  for neural networks with piecewise polynomial activation functions. In particular, their results showed that for ReLU networks with width $ W $ and depth $L$, it holds
\begin{equation}\label{pdim}
\Pdim(\NN(W,L)) \lesssim \min\{W^2 L^2 \log(W L), W^3 L^2\}.
\end{equation}
It is also well-known that metric entropy can also be bounded by pseudo-dimension \citep[Theorem 2.6.7]{vaart2023weak}. To be precise, we recall that an envelope function of a real-valued function class $\mathcal{H}$ defined on $\mathcal{X}$ is any function $H(x)$ such that $|h(x)| \leq H(x)$ for every $x \in \mathcal{X}$ and $h \in \mathcal{H}$. For any $p \geq 1$ and probability measure $\mu$ on $\mathcal{X}$, if $0 < \|H\|_{L^p(\mu)} < \infty$, then 
\begin{equation}\label{pseudo}
\cN(\varepsilon,\cH, \|\cdot\|_{L^p(\mu)}) \le C \Pdim(\cH) (16e)^{\Pdim(\cH)} \left( \frac{\|H\|_{L^p(\mu)}}{\varepsilon}\right)^{p \Pdim(\cH)},
\end{equation}
for a universal constant $C >0$ and any $0<\varepsilon<\|H\|_{L^p(\mu)}$. We can now derive the lower bound in Theorem \ref{thm:app lower bound 2} using similar argument as the proof of Theorem \ref{thm:app lower bound 1}.

\begin{proof}[Proof of Theorem \ref{thm:app lower bound 2}]
Let $\varepsilon>0$ be arbitrarily small and denote the approximation error by
\[
\delta = \sup_{f\in \cB^s_1(\Omega)} \inf_{g\in \NN(W,L)} \|f-g\|_{L^2(\Omega)} + \varepsilon.
\]
Let $f_1,\dots,f_{N_R}\in \mathbb B_1^s(\Omega)$ be the functions in Lemma \ref{lem:barron-separated-set} so that $N_R\ge 2^{cR^d}$ and
$$
\|f_i-f_j\|_{L^2(\Omega)} > c_0 R^{-s-d/2} =: \varepsilon_0,
\qquad 
\forall i\ne j.
$$
By the definition of approximation error, for each $i=1,\dots,N_R$, there exists a neural network $g_i \in \NN(W,L)$ such that $\|f_i - g_i\|_{L^2(\Omega)} \le \delta$. Recall that the truncation operator $\cT_M$ is defined by (\ref{trancation}) and $\|f\|_{L^\infty(\Omega)} \le \|f\|_{\cB^s(\Omega)} \le 1$ for any $f\in \cB^s_1(\Omega)$. It is easy to see that
$$
\|f_i - \cT_1 g_i\|_{L^2(\Omega)} \le \|f_i - g_i\|_{L^2(\Omega)} \le \delta.
$$

We are going to show that $\delta\ge \varepsilon_0/3$ for some $R$ depending on $W$ and $L$. Suppose on the contrary that $\delta< \varepsilon_0/3$. Then, for $i\ne j$,
$$
\begin{aligned}
	\| \cT_1 g_i - \cT_1 g_j\|_{L^2(\Omega)}
	&\ge \|f_i - f_j\|_{L^2(\Omega)}
	- \|f_i - \cT_1 g_i\|_{L^2(\Omega)}
	- \|f_j - \cT_1 g_j\|_{L^2(\Omega)} \\
	&> \varepsilon_0 - 2\delta > \varepsilon_0/3.
\end{aligned}
$$
Hence, $\{\cT_1 g_i\}_{i=1}^{N_R}$ is an $\varepsilon_0/3$-packing of the class $\cT_1 \NN(W,L)$ in $L^2(\Omega)$. Thus, we get a lower bound on the packing number
\[
\mathcal{M}\left(\varepsilon_0/3, \cT_1 \NN(W,L), \|\cdot\|_{L^2(\Omega)}\right) \ge N_R \ge 2^{cR^d}. 
\]
To obtain an upper bound, we denote $P:=\Pdim(\cT_1\NN(W,L))$ as the pseudo-dimension of the truncated network class and notice that truncation does not increase the pseudo-dimension \citep[Theorem 11.3]{anthony2009neural} so that the bound (\ref{pdim}) also holds for the class $\cT_1\NN(W,L)$. By choosing the envelop function $H(x)\equiv 1$, the inequality (\ref{pseudo}) with $\cH = \cT_1\NN(W,L)$ implies
\begin{align*}
2^{cR^d} \le \mathcal{M}(\varepsilon_0/3, \cT_1\NN(W,L), \| \cdot \|_{L^2(\Omega)})
&\le \mathcal{N}(\varepsilon_0/6, \cT_1\NN(W,L), \| \cdot \|_{L^2(\Omega)}) \\
&\le C P (16e)^P \left(\frac{6}{\varepsilon_0}\right)^{2P}.
\end{align*}
As a consequence, we have
\[
P \gtrsim R^d (\log(1/\varepsilon_0))^{-1} \gtrsim R^d(\log R)^{-1}.
\]
It is possible to choose 
\[
R\asymp (P\log P)^{1/d} \lesssim \left(\min\{W^2 L^2 \log(W L), W^3 L^2\} \log(WL)\right)^{1/d},
\]
such that the above inequality is always false and hence we get a contradiction for the assumption $\delta<\varepsilon_n/3$. With this choice of $R$,
\begin{align*}
\delta &\ge \varepsilon_n/3 \gtrsim R^{-s-\frac{d}{2}} \\
&\gtrsim \min\left\{W^2 L^2 \log(W L), W^3 L^2\right\}^{^{-\frac{1}{2}-\frac{s}{d}}} (\log (WL))^{^{-\frac{1}{2}-\frac{s}{d}}},
\end{align*}
which finishes the proof by the definition of $\delta$.
\end{proof}

\section{Generalization error} \label{gen proof}
In this section, we analyze the generalization error of the least squares estimator with deep ReLU networks and establish the minimax optimal rate for learning spectral Barron functions. Subsections \ref{rate proof} and \ref{optimality} give the proofs of Theorems \ref{thm:rate} and \ref{thm:minimax lower bound} respectively.

\subsection{Convergence rates}\label{rate proof}

We prove Theorem \ref{thm:rate} by using similar analysis as \citep{nakada2020adaptive,schmidthieber2020nonparametric,yang2024nonparametric}. Let us begin with a decomposition of the excess risk $\|\cT_1 \widehat{f}_n - f_0 \|^2_{L^2(\mu)}$. Let $f^*\in \cH_n$ be a good approximation of $f_0$, where $f^*$ will be chosen using Theorem \ref{thm:app upper bound}. We have
\begin{equation}\label{triangle ineq 1}
	\begin{aligned}
		\|\cT_1 \widehat{f}_n - f_0 \|^2_{L^2(\mu)}  
		&\leq 2 \|\cT_1 \widehat{f}_n - \cT_1 f^* \|^2_{L^2(\mu)}  +2\| \cT_1 f^* - f_0 \|^2_{L^2(\mu)}  \\
		&\leq 2\|\cT_2 (\widehat{f}_n - f^*) \|^2_{L^2(\mu)}  +2\| f^* - f_0 \|^2_{L^2(\mu)}  
	\end{aligned}
\end{equation}
where we use $|\cT_1\widehat{f}_n(x)-\cT_1f^*(x)| \le |\cT_{2} (\widehat{f}_n-f^*)(x)|$ and $|f_0(x)| \le 1$ for any $x\in \Omega$ in the second inequality. For convenience, we denote the following empirical $L^2$ norm on the training data
\[
\|f\|_n^2 := \frac{1}{n} \sum_{i=1}^n |f(X_i)|^2.
\]
Note that the second term on the bound (\ref{triangle ineq 1}) can be estimated by using Theorem \ref{thm:app upper bound}, while the first term can be bounded by its empirical counterpart $\|\cT_{2}(\widehat{f}_n-f^*)\|_{n}^2$ using statistical learning theory. The error bound depends on certain complexity of the model. Following \citet{yang2024nonparametric}, we use the following complexities with localization technique \citep{bartlett2005local,koltchinskii2006local}. 

\begin{definition}[Local complexities]
	Let \(\xi_{1:n}=(\xi_1,\dots,\xi_n)\) be a sequence of independent zero-mean random variables. For a given radius \(\delta > 0\) and a sequence of sample points \(X_{1:n}=(X_1,\dots,X_n)\), we define the local complexity of a function class \(\mathcal{H}\) at scale \(\delta\) with respect to \(\xi_{1:n}\) by
	\begin{equation*}
		\cG_n(\mathcal{H}; \delta, \xi_{1:n}) := \bE_{\xi_{1:n}} \left[ \sup_{f \in \mathcal{H}, \|f\|_{n} \leq \delta} \left| \frac{1}{n} \sum_{i=1}^n \xi_i f(X_i) \right| \right].
	\end{equation*}
	If each \(\xi_i\) is the Rademacher random variable (taking values \(\pm 1\) with equal probability \(1/2\)), \(\cG_n(\mathcal{H}; \delta, \xi_{1:n})\) is the local Rademacher complexity and is denoted by \(\mathcal{R}_n(\mathcal{H}; \delta)\).
\end{definition}

To use these local complexities, we often require that the function class $\mathcal{H}$ is star-shaped (around the origin), meaning that, for any $f \in \mathcal{H}$ and $ a \in [0,1]$, the function $af \in \mathcal{H}$. If the star-shaped condition fails to hold, one can replace it by its star hull
$$
\star(\mathcal{H}) := \{af:f\in \mathcal{H},a\in  [0,1]\}.
$$
Note that the neural network class $\NN(W,L,S,B)$ is already star-shaped, but its truncation $\cT_1\NN(W,L,S,B)$ may not. We can bound $\|\cT_2 (\widehat{f}_n - f^*) \|^2_{L^2(\mu)}$ by $\|\cT_{2}(\widehat{f}_n-f^*)\|_{n}^2$ using the following lemma from \citet[Theorem 14.1 and Proposition 14.25]{wainwright2019high}.

\begin{lemma}\label{uniform law}
	Given a star-shaped and $M$-uniformly bounded function class $\cH$, let $\varepsilon_n$ be any positive solution to the inequality
	\[
	\cR_n(\cH;\varepsilon_n) \le \frac{\varepsilon_n^2}{M}.
	\]
	Then, there are constants $c_1,c_2>0$ such that the bound  
	\[
	\|f\|_{L^2(\mu)}^2 \le 2 \|f\|_{n}^2 + \varepsilon_n^2, \quad \forall f\in \cH,
	\]
	holds with probability at least $1-c_1 \exp(-c_2 n\varepsilon_n^2/M^2)$.
\end{lemma}

The above lemma helps us quantify the effect of the samples $X_{1:n}$. It remains to bound the empirical error $\|\cT_{2}(\widehat{f}_n-f^*)\|_{n}^2$, which also depends on the noises $\eta_{1:n}=(\eta_1,\dots,\eta_n)$ for fixed $X_{1:n}$. We observe that 
\begin{equation}\label{triangle ineq 2}
\|\cT_{2}(\widehat{f}_n-f^*)\|_{n}^2 \le \|\widehat{f}_n-f^*\|_{n}^2 \le 2\|\widehat{f}_n-f_0\|_{n}^2 + 2\|f^*-f_0\|_{n}^2. 
\end{equation}
Again, the second term in the above bound can be estimated using Theorem \ref{thm:app upper bound}. For the first term, we use the definition of $\widehat{f}_n$, which implies 
\[
\frac{1}{n} \sum_{i=1}^n (\widehat{f}_n(X_i)- Y_i)^2 \le \frac{1}{n} \sum_{i=1}^n (f(X_i)- Y_i)^2,\quad \forall f\in \cH_n.
\]
Using $Y_i = f_0(X_i) + \eta_i$, we have the base inequality
\begin{equation}\label{base inequality}
\|\widehat{f}_n-f_0\|_{n}^2 \le \|f-f_0\|_{n}^2 + \frac{2}{n} \sum_{i=1}^n \eta_i \left(\widehat{f}_n(X_i)- f(X_i) \right),\quad \forall f\in \cH_n.
\end{equation}
The next lemma from \citet[Lemma 6]{yang2024nonparametric} shows that $\|\widehat{f}_n-f_0\|_{n}^2$ can be bounded by using the approximation error $\|f^*-f_0\|_n^2$ and the local complexity of the model. Note that this lemma requires the boundedness of noises. We will show how to apply it to unbounded noises that satisfy (\ref{noise assumption}) in the proof of Theorem \ref{thm:rate}.

\begin{lemma}\label{oracle inequality}
Assume that the noises $\|\eta_{1:n}\|_\infty\le M$ are bounded. For any fixed sample points $X_{1:n}$, let $\widehat{f}_n$ be any estimator that satisfies the inequality (\ref{base inequality}) and $\delta_n$ be any positive solution to the inequality
\[
\cG_n(\star(\partial \cH_n);\delta_n,\eta_{1:n}) \le \delta_n^2,
\]
where  $\partial \cH_n = \{f-g:f,g\in \cH_n\}$. Then, there are constants $c_1,c_2>0$ such that the bound
\[
\|\widehat{f}_n-f_0\|_n^2 \le 3 \|f-f_0\|_n^2 + 32 \delta_n^2, \quad \forall f\in \cH_n, 
\]
holds with probability at least $1-c_1 \exp(-c_2 n\delta_n^2/M^2)$.
\end{lemma}

In order to apply the above analysis to the neural network class $\NN(W,L,S,B)$, we need to estimate its local complexity, which is given in the following Theorem.

\begin{theorem}\label{local Gc bound}
	Let $\xi_{1:n}=(\xi_1,\dots,\xi_n)$ be a sequence of independent sub-Gaussian random variables with parameter $\varsigma >0$ in the sense that
	$$
	\bE[\exp(\lambda \xi_i)] \leq \exp( \varsigma^2 \lambda^2 / 2), \quad \forall \lambda \in \mathbb{R}.
	$$
	Then, for any $0 < \delta \leq 1$,
	$$
	\cG_n(\NN(W,L,S,B); \delta, \xi_{1:n}) \lesssim \varsigma \delta  \frac{\sqrt{SL}}{\sqrt{n}} \log(nSB/\delta),
	$$
	where the implied constant is independent of $\xi_{1:n}$ and the sample points $X_{1:n}$ in $\Omega$. The bound also holds for the function class $\star(\cT_M \NN(W,L,S,B))$ for any $M\lesssim (SB)^L$.
\end{theorem}
\begin{proof}
We follow the chaining argument as the proof of \citet[theorem 5]{yang2024nonparametric}. Without loss of generality, we can assume that $\varsigma=1$ by rescaling. It is well known that the Gaussian complexity can be bounded by Dudley's entropy integral \citep[Section 5.3.3]{wainwright2019high}. For convenience, we denote $\cH= \NN(W,L,S,B)$ and $\cH_\delta= \{f\in \cH: \|f\|_n\le \delta\}$. Let $\cN(\varepsilon,\cH_\delta,\|\cdot\|_n)$ be the $\varepsilon$-covering number of $\cH_\delta$ in the empirical norm $\|\cdot\|_n$. \citet[Eq. (30)]{yang2024nonparametric} showed that
\begin{align*}
\cG_n(\cH; \delta, \xi_{1:n}) &= 
\bE_{\xi_{1:n}} \left[ \sup_{f \in \cH_\delta} \left| \frac{1}{n} \sum_{i=1}^n \xi_i f(X_i) \right| \right] \\
&\leq \inf_{\varepsilon \geq 0} \left\{ 4\varepsilon + \frac{16}{\sqrt{n}} \int_\varepsilon^{D/2} \sqrt{\log \cN(t, \cH_\delta, \|\cdot\|_{n})} dt \right\},
\end{align*}
where $D := \sup_{f, f' \in \cH_\delta} \|f - f'\|_{n}\le 2\delta$ denotes the diameter. 

Since $\|\cdot\|_n\le \|\cdot\|_{L^\infty(\Omega)}$, by the estimate (\ref{nn up covering}), we have
\[
\log \cN(t, \cH_\delta, \|\cdot\|_{n}) \le S \log\left(\frac{2LB^L(S+1)^L}{\varepsilon}\right) \lesssim SL \log \left(\frac{SB}{\varepsilon}\right),
\]
where we use the fact that $L\le S$ in the last inequality. Applying this bound to the above entropy integral bound, we get
\begin{align*}
\cG_n(\cH; \delta, \xi_{1:n}) &\lesssim \inf_{\varepsilon \geq 0} \left\{ \varepsilon + \frac{\sqrt{SL}}{\sqrt{n}} \int_\varepsilon^{\delta} \log(SB/t) dt \right\} \\
&\lesssim \inf_{\varepsilon \geq 0} \left\{ \varepsilon + \frac{\sqrt{SL}}{\sqrt{n}} \delta \log(BS/\varepsilon) \right\}.
\end{align*}
If we choose $\varepsilon = \sqrt{SL} \delta/\sqrt{n}$, then
\[
\cG_n(\cH; \delta, \xi_{1:n}) \lesssim \frac{\sqrt{SL}}{\sqrt{n}} \delta \log(nSB/\delta),
\]
which is the desired inequality.

Finally, for the star hull $\star(\cT_M \cH)$, we can similarly derive the bound by estimating the covering number and using the entropy integral bound. Specifically, we observe that the covering number of $\cT_M \cH$ is not larger than that of $\cH$. In order to estimate the covering number of $\star(\cT_M \cH)$, for any $\varepsilon>0$, we let $\{f_1,\dots,f_m\} \subseteq \cT_M\cH$ be an $\varepsilon/2$-cover of $\cT_M\cH$ and $\{a_1,\dots,a_k\} \subseteq [0,1]$ be an $\varepsilon/(2M)$-cover of $[0,1]$. Then, it is easy to show by the triangle inequality that $\{a_if_j:1\le i\le k, 1\le j\le m\}$ is an $\varepsilon$-cover of $\star(\cT_M \cH)$, which implies
\begin{align*}
&\log \cN(\varepsilon, \star(\cT_M \cH), \|\cdot\|_{L^\infty(\Omega)}) \\
\le &\ \log \cN(\varepsilon, \cH, \|\cdot\|_{L^\infty(\Omega)}) + \log (1+M/\varepsilon) \\
\lesssim &\ SL \log \left(\frac{SB}{\varepsilon}\right),
\end{align*}
where we use the estimate (\ref{nn up covering}) and $M\lesssim (SB)^L$ in the last inequality. Hence, we can obtain the same bound for the local complexity.
\end{proof}

We can now prove Theorem \ref{thm:rate}. In the proof, we will use $c_1,c_2, \dots,$ to denote positive constants for convenience.

\begin{proof}[Proof of Theorem \ref{thm:rate}] For any regression function $f_0\in \cB^s_1(\Omega)$, we can choose $f^*\in \cH_n=\NN(W_n,L_n,S_n,B_n)$ as Theorem \ref{thm:app upper bound} such that  
\[
\|f^* - f_0\|_{L^\infty(\Omega)} \le \varepsilon_{app},
\]
where $\varepsilon_{app}$ will be chosen later. By the error decomposition (\ref{triangle ineq 1}), 
\begin{equation}\label{error dec}
\|\cT_1 \widehat{f}_n - f_0 \|^2_{L^2(\mu)} \leq 2\|\cT_2 (\widehat{f}_n - f^*) \|^2_{L^2(\mu)}  +2\varepsilon_{app}^2.
\end{equation}
We divide the proof into three steps. 

\textbf{Step 1.} Bounding $\|\cT_2 (\widehat{f}_n - f^*) \|^2_{L^2(\mu)}$ by $\|\cT_{2}(\widehat{f}_n-f^*)\|_{n}^2$. Since the function $\widehat{f}_n - f^* \in \partial \cH_n \subseteq \NN(2W_n, L_n, 2S_n, B_n)$, its truncation $\cT_2 (\widehat{f}_n - f^*)$ is in the star-shaped and $2$-uniformly bounded class $\star(\cT_2\NN(2W_n, L_n, 2S_n, B_n))$. By Theorem \ref{local Gc bound}, its local Rademacher complexity satisfies
\[
\cR_n(\star(\cT_2\NN(2W_n,L_n, 2S_n, B_n)); \varepsilon_n) \lesssim \varepsilon_n \frac{\sqrt{S_nL_n}}{\sqrt{n}} \log(nS_nB_n/\varepsilon_n).
\]
By choosing suitable
\[
\varepsilon_n \asymp \frac{\sqrt{S_nL_n}}{\sqrt{n}} \log(nS_nB_n),
\]
we can make sure that the local Rademacher complexity is upper bounded by $\varepsilon_n^2/2$, so that we can apply Lemma \ref{uniform law} to the class $\star(\cT_2\NN(2W_n, L_n, 2S_n, B_n))$ and conclude that the inequality
\begin{equation}\label{step 1 bound}
\|\cT_2(\widehat{f}_n - f^*)\|_{L^2(\mu)}^2 \le 2\|\cT_2(\widehat{f}_n - f^*)\|_{n}^2 + \varepsilon_n^2
\end{equation}
holds with probability at least $1-c_1 \exp(-c_2 n\varepsilon_n^2)$.

\textbf{Step 2.} Bounding $\|\cT_{2}(\widehat{f}_n-f^*)\|_{n}^2$. Since the noises $\eta_{1:n}$ are not uniformly bounded, we introduce a threshold $M_n \ge 1$, which will be chosen later, and define $\xi_{1:n}=(\xi_1, \dots,\xi_n)$ with $\xi_i=\cT_{M_n}\eta_i$ and the event $\cA_n := \{\eta_{1:n} = \xi_{1:n}\}$. Notice that $\eta_i = \xi_i$ if and only if $|\eta_i|\le M_n$. By the noise assumption \eqref{noise assumption}, we have
\begin{align*}
\bP(\cA_n) &= 1- \bP \left(\exists i\in \{1,\dots,n\} \mbox{ s.t. } |\eta_i|>M_n \right) \ge 1- \sum_{i=1}^{n} \bP(|\eta_i|>M_n) \\
&\ge 1- n \exp (-(M_n/c_\eta)^{1/q}).
\end{align*}
For any $\rho\ge 1$, we choose $M_n = c_\eta (\rho+\log n)^q$ so that $\bP(\cA_n) \ge 1-\exp(-\rho)$. In the following analysis, we condition on the event $\cA_n$ so that the noises are bounded by $M_n$ and hence they are sub-Gaussian with parameter $\varsigma = M_n$, see \citep[Example 2.4]{wainwright2019high} for instance. By Theorem \ref{local Gc bound}, since $\star(\partial \cH_n) \subseteq \NN(2W_n,L_n, 2S_n, B_n)$, we have
\[
\cG_n(\star(\partial \cH_n); \delta_n, \xi_{1:n}) \lesssim \delta_n M_n \frac{\sqrt{S_nL_n}}{\sqrt{n}} \log(nS_nB_n/\delta_n).
\]
By choosing suitable
\[
\delta_n \asymp M_n \frac{\sqrt{S_nL_n}}{\sqrt{n}} \log(nS_nB_n),
\]
we are guaranteed that the local complexity is upper bounded by $\delta_n^2$. Thus, we can apply Lemma \ref{oracle inequality} to the function class $\cH_n$, which implies that the inequality
\[
\|\widehat{f}_n-f_0\|_n^2 \le 3 \|f^*-f_0\|_n^2 + 32 \delta_n^2 \le 3 \varepsilon_{app}^2 + 32 \delta_n^2
\]
holds with probability at least $1-c_3 \exp(-c_4 n\delta_n^2/M_n^2)$. Combining the above analysis with the inequality (\ref{triangle ineq 2}), we conclude that
\begin{equation}\label{step 2 bound}
\|\cT_{2}(\widehat{f}_n-f^*)\|_{n}^2 \le 8 \varepsilon_{app}^2 + 64 \delta_n^2
\end{equation}
holds with probability at least $1- \exp(-\rho)- c_3 \exp(-c_4 n\delta_n^2/M_n^2)$.

\textbf{Step 3.} Optimizing the parameters. Combining the inequalities (\ref{error dec}), (\ref{step 1 bound}) and (\ref{step 2 bound}), we have
\[
\|\cT_1 \widehat{f}_n - f_0 \|^2_{L^2(\mu)} \le 34 \varepsilon_{app}^2 + 2\varepsilon_n^2 + 256\delta_n^2,
\]
which holds with probability at least $1- \exp(-\rho)- c_1 \exp(-c_2 n\varepsilon_n^2) - c_3 \exp(-c_4 n\delta_n^2/M_n^2)$. Notice that, by our choices of $\varepsilon_n$ and $\delta_n$, we always have $\varepsilon_n \lesssim \delta_n$ and $\varepsilon_n^2 \asymp \delta_n^2/M_n^2$. Thus, 
\[
\|\cT_1 \widehat{f}_n - f_0 \|^2_{L^2(\mu)} \lesssim \varepsilon_{app}^2 + \frac{M_n^2 S_nL_n}{n} (\log(nS_nB_n))^2,
\]
which holds with probability at least $1- \exp(-\rho)- c_5 \exp(-c_6 S_nL_n (\log(nS_nB_n))^2)$. By Theorem \ref{thm:app upper bound}, we can choose
\begin{align*}
\varepsilon_{app} &\asymp n^{-\frac{d+2s}{4d+4s}} (\log n)^{\frac{7d+12s}{4d+4s}}\\
W_n &\asymp n^{\frac{d}{2d+2s}} (\log n)^{-\frac{5d}{2d+2s}}\\
L_n &\asymp (\log n)^2 \\
S_n &\asymp n^{\frac{d}{2d+2s}} (\log n)^{\frac{4s-d}{2d+2s}}\\
B_n &\asymp 1.
\end{align*}
Then, it holds with probability at least $1- \exp(-\rho)- c_5 \exp(-c_6 n^{\frac{d}{2d+2s}} (\log n)^{\frac{7d+12s}{2d+2s}})$ that
\[
\|\cT_1 \widehat{f}_n - f_0 \|^2_{L^2(\mu)} \lesssim M_n^2 n^{-\frac{d+2s}{2d+2s}} (\log n)^{\frac{7d+12s}{2d+2s}}.
\]
Recalling that $M_n = c_\eta (\rho+\log n)^q$, we finish the proof.
\end{proof}

\subsection{Optimality}\label{optimality}
In this section, we establish a minimax lower bound for estimating $f_0 \in \cB_1^s(\Omega)$ under the $L^2(\Omega)$ norm, which proves Theorem \ref{thm:minimax lower bound}. The result shows that the convergence rate of the least squares estimator derived in the previous section is optimal up to logarithmic factors.

To apply the information-theoretic machinery based on Fano's inequality and the Yang-Barron method \citep{wainwright2019high,yang1999information}, we first recast the model in terms of the joint distribution of the data.
Recall that, in Theorem \ref{thm:minimax lower bound}, the training samples $\cD_n =\{(X_i, Y_i)\}_{i=1}^n$ satisfies
$$
Y_i = f(X_i) + \eta_i,\quad X_i \sim \mu,\quad \eta_i \stackrel{\mathrm{i.i.d.}}{\sim} \mathcal{N}(0,\sigma^2),
$$
for some $f \in \cB^s_1(\Omega)$ and $\mu$ is the uniform distribution on $\Omega$. Let $\bQ = \mu^{\otimes n}$ be the distribution of the covariates
$(X_1,\dots,X_n)$. For any $f \in \cB_1^s(\Omega)$, let
$\bP_f$ denote the conditional distribution of $(Y_1,\dots,Y_n)$
given $(X_1,\dots,X_n)$,
$$
\bP_f = \mathcal{N}(f(X_1),\sigma^2) \times \cdots	\times \mathcal{N}(f(X_n),\sigma^2).
$$
The full joint distribution of the observed data $\cD_n$ is therefore $\bP_f \times \bQ$.

To derive minimax lower bound, we consider the following hypothesis testing problem defined by the family of distributions $\{\bP_{f_j} \times \bQ: j=1,\dots,M \}$, where $f_1, \dots, f_M \in \cB^s_1(\Omega)$ such that $\|f_i-f_j\|_{L^2(\Omega)} \ge 2\delta$ for any $i\neq j$. We generate a random variable $Z$ by first sample a random integer $J$ from the uniform distribution over $\{1,\dots,M\}$ and then condition on $J=j$, sample $Z$ from $\bP_{f_j} \times \bQ$. A testing function $\Psi$ for this problem is a mapping that assigns $Z$ to one of the index in $\{1,\dots,M\}$. Let $\bP_{Z,J}$ denotes the joint probability distribution of $(Z,J)$. The error probability of the testing functions gives a lower bound for the minimax rate \citep[Proposition 15.1]{wainwright2019high},
\begin{equation}\label{minimax test bound}
\inf_{\widehat{f}} \sup_{f_0 \in \cB_1^s(\Omega)} \bE_{\cD_n}\left[ \|\widehat{f} - f_0\|_{L^2(\Omega)}^2 \right] \ge \delta^2 \inf_{\Psi} \bP_{Z,J}(\Psi(Z) \neq J),
\end{equation}
where the infimum ranges over testing functions. Fano's inequality lower bound the above error probability by using the mutual information between $Z$ and $J$. Recall that, for two probability measures $\mu$ and $\nu$, which are absolutely continuous with respect to some probability measure $\tau$ on $\cX$, their Kullback-Leibler divergence is defined as
\[
D_{\mathrm{KL}} (\mu \| \nu) = \int_\cX p(x) \log \frac{p(x)}{q(x)} d\tau(x),
\]
where $p(x)$ and $q(x)$ are the densities of $\mu$ and $\nu$ respectively. Note that the random variables $Z$ and $J$ are independent if and only if their joint distribution $\bP_{Z,J}$ is equal to the product of its marginals, namely, $\bP_Z \bP_J$. Thus, we can measure the their dependence by using the KL-divergence of these two measures
\[
I(Z;J) = D_{\mathrm{KL}} (\bP_{Z,J} \| \bP_Z \bP_J),
\]
which is Shannon's mutual information between $Z$ and $J$. Fano's inequality shows that (\ref{minimax test bound}) can be lower bounded by \citep[Eq. (15.31)]{wainwright2019high}
\begin{equation}\label{fano}
\inf_{\Psi} \bP_{Z,J}(\Psi(Z) \neq J) \ge 1 - \frac{I(Z;J)+\log 2}{\log M}.
\end{equation}
Thus, it is sufficient to find an upper bound for the mutual information $I(Z;J)$. \citet{yang1999information} gave a systematical way to achieve this purpose. Let $\cP= \{\bP_f \times \bQ :f\in \cB^s_1(\Omega)\}$ and $\cN_{\mathrm{KL}}(\varepsilon, \cP)$ denotes the $\varepsilon$-covering number of $\cP$ in the square-root KL-divergence. Then $I(Z;J)$ can be upper bounded as \citep[Lemma 15.12]{wainwright2019high}
\begin{equation}\label{eq:yang barron}
I(Z;J) \le \inf_{\varepsilon>0} \left\{ \varepsilon^2 + \log \cN_{\mathrm{KL}}(\varepsilon,\cP) \right\}.
\end{equation}

We use the above framework to prove Theorem \ref{thm:minimax lower bound}, which gives a minimax lower bound for learning spectral Barron functions.

\begin{proof}[Proof of Theorem \ref{thm:minimax lower bound}]
For convenience, let us denote the $\varepsilon$-covering and $\varepsilon$-packing numbers of $\cB^s_1(\Omega)$ in $L^2(\Omega)$ norm by $\cN(\varepsilon)$ and $\cM(\varepsilon)$ respectively. Observe that, for any $f,g\in \cB^s_1(\Omega)$,
\begin{align*}
D_{\mathrm{KL}} (\bP_f \times \bQ \| \bP_g \times \bQ) &= \bE_{(X_1,\dots,X_n)\sim \bQ} \left[ D_{\mathrm{KL}}(\bP_f\| \bP_g) \right] \\
&= \bE_{(X_1,\dots,X_n)\sim \bQ} \left[ \frac{1}{2\sigma^2} \sum_{i=1}^n (f(X_i)-g(X_i))^2 \right] \\
&= \frac{n}{2\sigma^2} \|f-g\|_{L^2(\Omega)}^2.
\end{align*}
Thus, any $\sqrt{2}\sigma \varepsilon /\sqrt{n}$-cover in $L^2(\Omega)$ is an $\varepsilon$-cover in square-root KL-divergence and we have 
\[
\cN_{\mathrm{KL}}(\varepsilon,\cP) \le \cN(\sqrt{2}\sigma \varepsilon /\sqrt{n}).
\]

The covering number can be estimated by using inequality (\ref{temp2}) and Remark \ref{app remark}, which show that any functions in $\cB^s_1(\Omega)$ can be approximated by function $f_{N,T}$ of the form (\ref{f_NT}) with approximation error $\cE_N \asymp N^{-\frac12-\frac{s}{d}}$ and $T=\cE_N^{-1/s}$. Since the function $f_{N,T}$ is parameterized by $CN$ parameters in $[-T,T]$, it is easy to prove that the $\varepsilon$-covering number of these functions is $\cO((1+T/\varepsilon)^{CN})$ by discretizing these parameters. By letting $\varepsilon \asymp \cE_N \asymp N^{-\frac12-\frac{s}{d}}$, we get
\[
\log \cN(\varepsilon) \lesssim  \varepsilon^{-\frac{2d}{d+2s}} \log(1/\varepsilon).
\]
On the other hand, we can lower bound the packing number $\cM(\varepsilon)$ using Lemma \ref{lem:barron-separated-set}, which shows that, for every \(R\ge 1\), there exist $N_R\ge 2^{cR^d}$ functions $f_1,\dots,f_{N_R}\in \mathbb B_1^s(\Omega)$ such that $\|f_i-f_j\|_{L^2(\Omega)} > c_0 R^{-s-d/2}$ for all $i\neq j$. By letting $\varepsilon = c_0 R^{-s-d/2}$, we have
\[
\log \cM(\varepsilon) \gtrsim \varepsilon^{-\frac{2d}{d+2s}}.
\]

For large $n$, we can choose $\varepsilon_n \asymp n^{\frac{d}{4d+4s}} (\log n)^{\frac{d+2s}{4d+4s}}$ such that $\log \cN(\sqrt{2}\sigma \varepsilon_n /\sqrt{n}) \le \varepsilon_n^2$. By (\ref{eq:yang barron}),
\[
I(Z;J) \le \inf_{\varepsilon>0} \left\{ \varepsilon^2 + \log \cN(\sqrt{2}\sigma \varepsilon_n /\sqrt{n}) \right\} \le 2 \varepsilon_n^2.
\]
We can choose $\delta_n \asymp \varepsilon_n^{-\frac{d+2s}{d}} \asymp n^{-\frac{d+2s}{4d+4s}} (\log n)^{-\frac{d+2s}{4d+4s} \frac{d+2s}{d}}$ such that $\log \cM(2\delta_n) \ge 4 \varepsilon_n^2 + 2 \log 2$. Then, inequality (\ref{fano}) with $M=\cM(2\delta_n)$ implies
\[
\inf_{\Psi} \bP_{Z,J}(\Psi(Z) \neq J) \ge 1 - \frac{2 \varepsilon_n^2+\log 2}{\log \cM(2\delta_n)} \ge \frac{1}{2}.
\]
Finally, using (\ref{minimax test bound}), we have
\[
\inf_{\widehat{f}} \sup_{f_0 \in \cB_1^s(\Omega)} \bE_{\cD_n}\left[ \|\widehat{f} - f_0\|_{L^2(\Omega)}^2 \right] \ge \frac{1}{2} \delta_n^2 \gtrsim \left(n(\log n)^{1+\frac{2s}{d}}\right)^{-\frac{d+2s}{2d+2s}},
\]
which completes the proof.
\end{proof}

\section*{Acknowledgments}

The work described in this paper is partially supported by National Natural Science Foundation of China under Grants 12501131 and 12526216.

\bibliographystyle{myplainnat}
\bibliography{References}
\end{document}